\documentclass{article}

\usepackage{microtype}
\usepackage{graphicx}
\usepackage{subfigure}
\usepackage{booktabs} 

\usepackage{hyperref}

\usepackage[accepted]{icml2024}

\usepackage{amsmath}
\usepackage{amssymb}
\usepackage{mathtools}
\usepackage{amsthm}
\usepackage{bm}

\usepackage{enumitem}
\setitemize{noitemsep,topsep=0pt,parsep=0pt,partopsep=0pt}

\usepackage[capitalize,noabbrev]{cleveref}
\usepackage{multirow}

\theoremstyle{plain}
\newtheorem{theorem}{Theorem}[section]
\newtheorem{proposition}[theorem]{Proposition}
\newtheorem{lemma}[theorem]{Lemma}

\theoremstyle{definition}
\newtheorem{definition}[theorem]{Definition}

\theoremstyle{remark}

\newenvironment{hproof}{%
  \proof}{\endproof}

\usepackage[textsize=tiny]{todonotes}

\icmltitlerunning{LoRA Dropout as a Sparsity Regularizer for Overfitting Control}

\begin{document}

\twocolumn[
\icmltitle{LoRA Dropout as a Sparsity Regularizer for Overfitting Control}



\icmlsetsymbol{equal}{*}

\begin{icmlauthorlist}
\icmlauthor{Yang Lin}{pku,lab,equal}
\icmlauthor{Xinyu Ma}{pku,lab,equal}
\icmlauthor{Xu Chu}{pku,lab}
\icmlauthor{Yujie Jin}{pku,lab}
\icmlauthor{Zhibang Yang}{lab}
\icmlauthor{Yasha Wang}{pku,lab}
\icmlauthor{Hong Mei}{pku}
\end{icmlauthorlist}

\icmlaffiliation{pku}{Peking University, Beijing, China}
\icmlaffiliation{lab}{Key Laboratory of High Confidence Software Technologies, Ministry of Education, Beijing, China}
\icmlcorrespondingauthor{Xu Chu}{chu\_xu@pku.edu.cn}
\icmlcorrespondingauthor{Yasha Wang}{wangyasha@pku.edu.cn}

\icmlkeywords{Machine Learning, ICML}

\vskip 0.3in
]



\printAffiliationsAndNotice{\icmlEqualContribution} 

\begin{abstract}
Parameter-efficient fine-tuning methods, represented by LoRA, play an essential role in adapting large-scale pre-trained models to downstream tasks. 
However, fine-tuning LoRA-series models also faces the risk of overfitting on the training dataset, and yet there's still a lack of theoretical guidance and practical mechanism to control overfitting on LoRA-based PEFT methods. 
In this paper, we propose a LoRA Dropout mechanism for the LoRA-based methods by introducing random noises to the learnable low-rank matrices and increasing parameter sparsity. 
We then demonstrate the theoretical mechanism of our LoRA Dropout mechanism from the perspective of sparsity regularization by providing a generalization error bound under this framework. Theoretical results show that appropriate sparsity would help tighten the gap between empirical and generalization risks and thereby control overfitting. 
Furthermore, based on the LoRA Dropout framework, we introduce a test-time ensemble strategy and provide theoretical evidence demonstrating that the ensemble method can further compress the error bound, and lead to better performance during inference time. 
Extensive experiments on various NLP tasks provide practical validations of the effectiveness of our LoRA Dropout framework in improving model accuracy and calibration.

\end{abstract}

\section{Introduction}
In recent years, Pre-trained Language Models (PLMs)~\cite{devlin2018bert,liu2019roberta,he2020deberta,touvron2023llama}  have demonstrated increasingly superior performances in various natural language processing tasks as the rapid growth of model parameter scale. However, with the increasing model capacity and complexity, the challenge arises when adapting the PLMs to specific downstream tasks, as fully fine-tuning often requires substantial computational resources. Therefore, a new fine-tuning paradigm emerges named Parameter-Efficient Fine-Tuning (PEFT), aiming to optimize and adapt PLMs to specific downstream tasks with minimal adjustments to their parameters.

There has been a long line of research in the field of PEFT~\cite{houlsby2019parameter,lester2021power,li2021prefix,hu2021lora,zhang2023adaptive,ma2024parameter}. Among these works, the Low-Rank Adaptation (LoRA) method \cite{hu2021lora} and its variants \cite{dettmers2023qlora, zhang2023adaptive,zi2023delta} have been the most effective and widely adopted. The basic idea behind LoRA is that only some zero-initialized delta weight matrices get optimized during fine-tuning, and the original pre-trained parameters remain unmodified. To improve parameter efficiency, LoRA further decomposes the delta weight matrix into the product of two low-rank matrices.

\begin{figure}[t]
\centering
  \includegraphics[width=3.3 in]{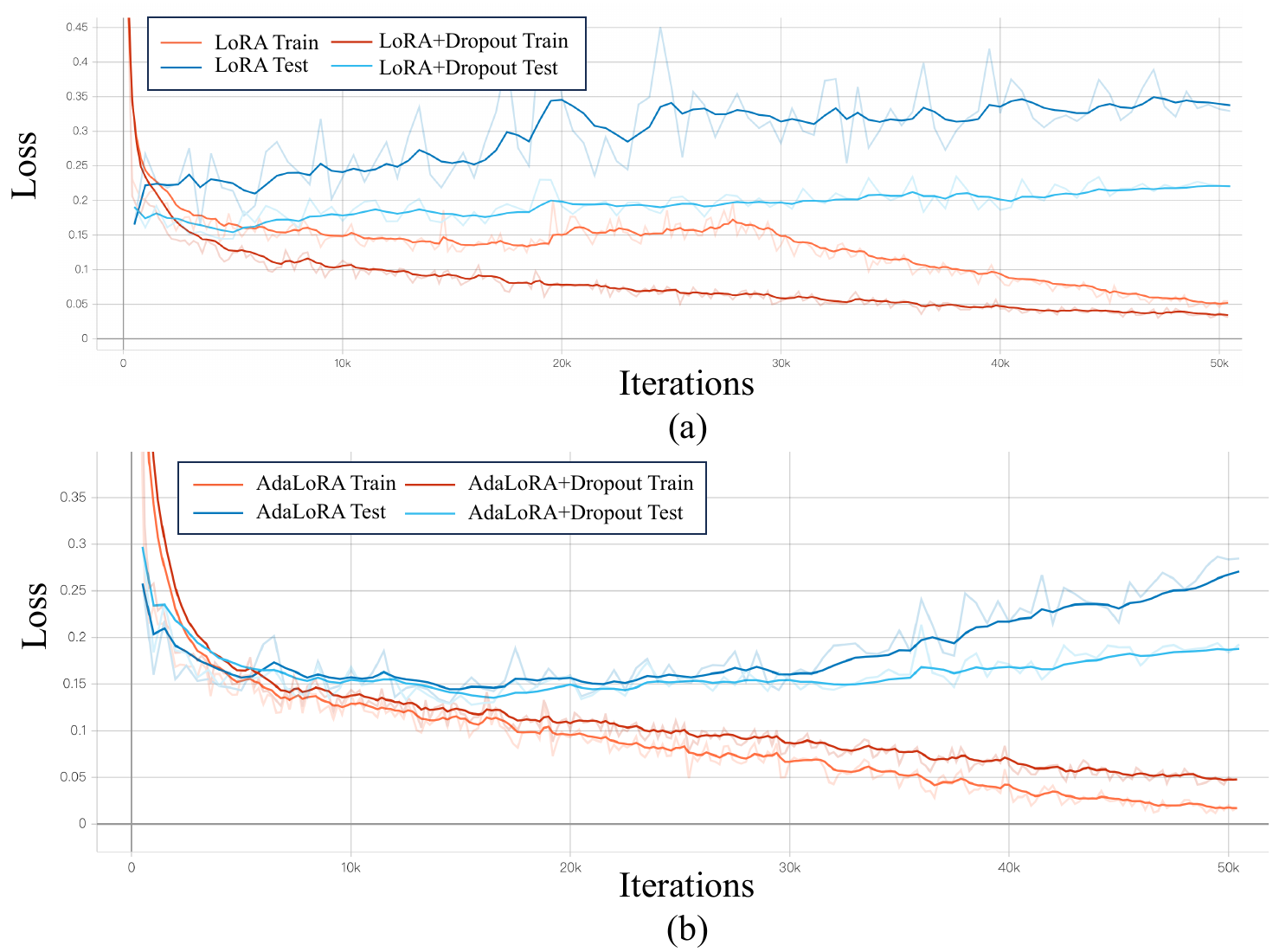}
  \caption{Loss curves on train and test set of SST2 dataset during fine-tuning of (a) LoRA w/wo our dropout framework, (b) AdaLoRA w/wo our dropout framework.}
  \label{fig:loss_curve}
  \vspace{-0.12in}
\end{figure}

However, for LoRA-series methods, picking a proper rank in the low-rank decomposition remains a significant challenge. In order to better adapt to the semantic shift between pre-trained and downstream tasks, these models typically tend to maintain a relatively high rank to ensure sufficient expressive power. However, an excessively high rank will increase the model's degrees of freedom, elevating the risk of overfitting downstream tasks. 
To overcome this challenge and select an appropriate parameter budget, one of LoRA's variants, AdaLoRA~\cite{zhang2023adaptive} proposes to automatically select fine-tuning parameters by learning the pseudo-singular values of the matrix decomposition. However, this parameter selection method heavily relies on gradients of parameters on the training set, which actually further increases the risk of overfitting the model to the training data. As shown in Figure.\ref{fig:loss_curve}, the gaps between train and test losses on both LoRA and AdaLoRA become larger during the fine-tuning process.
In fact, to the best of our knowledge, there is still a lack of theoretical guidance and practical mechanism to control overfitting on LoRA-based PEFT methods.

In this paper, motivated by the dropout regularization technique~\cite{hinton2012improving,srivastava2014dropout} that is commonly adopted in deep learning, we introduce a LoRA Dropout regularization framework on LoRA learnable parameters to control the overfitting risk when fine-tuning LoRA-based models. By introducing random noises and increasing sparsity on tunable parameters, our framework can improve the generalization ability of LoRA-based models on downstream tasks.
Furthermore, to gain a deeper insight into the above dropout framework, in this work, we answer the following profound question:

\textbf{ \textit{What is the theoretical mechanism behind the alleviation of overfitting on the training data through random dropout of LoRA parameters?}} 

We build our theoretical framework by modeling the training process under the LoRA Dropout from the perspective of sparse fine-tuning, and show that fine-tuning with LoRA Dropout can be viewed as an optimization problem with sparsity regularization. We further provide a generalization error bound under the sparsity regularization framework. Through this generalization error bound, we reveal that introducing appropriate sparsity on LoRA tunable parameters during fine-tuning helps to balance the empirical risk minimization and complexity of the adaptation function, thereby tightening the gap between empirical and generalization risks and controlling overfitting on the training data.


Besides the ability of LoRA Dropout to control overfitting during the fine-tuning period, we propose a test-time ensemble method to further improve the model's performance during the inference stage. By activating dropout during testing, we obtain an ensemble classifier consisting of models with different parameter dropouts.
Theoretical evidence demonstrates that the ensemble classifier can further compress the generalization error bound, and lead to a better test-time accuracy. 

In summary, we conclude the main contributions of this paper as follows. We propose a theoretically grounded dropout framework designed for LoRA-based models to promote their generalization ability on downstream tasks. Specifically, we provide theoretical analyses on the generalization error bound for LoRA Dropout, which balances the empirical risk minimization and complexity of the adaptation function class through sparsity regularization. Based on the LoRA Dropout framework, we propose an ensemble strategy during the inference stage. 
We show that this strategy will lead to an ensemble model with a tighter error bound and further enhance the model's test-time generalizability.
Extensive experiments conducted on a wide range of NLP tasks demonstrate the effectiveness of our method in improving the model's accuracy and calibration.

\section{Related Works}
\subsection{Parameter-Efficient Fine-Tuning (PEFT)}
With the increasing model size and complexity, Pre-trained Language Models (PLMs) demonstrate powerful performance across various NLP tasks, but it also becomes harder to efficiently adapt PLMs to downstream tasks. PEFT aims to solve this challenge by fine-tuning a few additional parameters or a subset of pre-trained parameters. Current mainstream PEFT approach can be roughly divided into three categories~\cite{lialin2023scaling,xu2023parameter}.
\textit{Additive Fine-tuning} methods~\cite{houlsby2019parameter,pfeiffer2020adapterfusion,he2021towards,lester2021power,li2021prefix} focus on adding extra tunable parameters by introducing additional layers or learnable prompts. \textit{Partial Fine-tuning} methods~\cite{zaken2022bitfit,xu2021raise,fu2023effectiveness} select a subset of pre-trained parameters for fine-tuning. \textit{Reparameterization Fine-tuning} methods~\cite{hu2021lora,zhang2023adaptive,edalati2022krona} adopt low-rank representations to minimize the number of trainable parameters. 
In this paper, we focus on the most effective and widely adopted method, LoRA~\cite{hu2021lora} and its variants, which decompose the learnable delta weight into the product of two low-rank matrices. The rank of the decomposition is essential for LoRA. A small rank may lead to insufficient expressive power, while a large rank could result in overfitting. One of LoRA’s variants, AdaLoRA~\cite{zhang2023adaptive} proposes to decompose the delta weight through a quasi-SVD method, and select parameters through importance scoring. Nevertheless, this selection method also relies on
gradients on the training set, leading to an additional risk of overfitting. In this work, we propose a  theoretically grounded dropout framework for LoRA-series methods, filling the gap that the LoRA-based PEFT methods lack theoretical guidances and practical mechanisms to control overfitting.

\subsection{Dropout Regularization}
The dropout mechanism~\cite{hinton2012improving} is a well-known and widely-adopted technique in deep neural networks to prevent overfitting. In standard dropout, each neuron in the network is omitted from the network with a certain possibility during training. Subsequently, various dropout techniques were introduced. Dropconnect~\cite{wan2013regularization} sets model weights rather than neuron outputs to zero with some probability. Gaussian dropout~\cite{srivastava2014dropout} replaces the Bernoulli noises in standard dropout with Gaussian noises. Dropout methods for specific model structures are also proposed, like Spatial dropout~\cite{tompson2015efficient} for convolutional layers and Recurrent dropout~\cite{semeniuta2016recurrent} for recurrent neural networks. Meanwhile, works have been done to explore the theoretical factors behind dropout's ability to suppress overfitting. Some works believe that the the model learns a geometric mean over the ensemble of possible sub-networks through dropout~\cite{warde2013empirical,baldi2013understanding}, and some works view dropout from a Bayesian perspective and argue that model with dropout can be interpreted as a Bayesian model approximating a posterior over parameters~\cite{gal2016dropout}. However, despite extensive research of the dropout technique, currently there's little practical or theoretical work on applying dropout on LoRA-based PEFT models, where fine-tuning happens on the delta weight matrices with low-rank decompositions.

\section{Proposed LoRA Dropout Framework}
In this section, we present our LoRA Dropout framework. We start by briefly reviewing the LoRA method. Then we show the details of our LoRA Dropout on the trainable low-rank parameter matrices. Finally, based on LoRA Dropout, we present our training objective for the fine-tuning phase and the ensemble strategy for the inference phase.

\subsection{Background: Low Rank Adaptation (LoRA)}
Before introducing our method, we give a brief review of the LoRA method~\cite{hu2021lora} in PEFT.
When fine-tuning on downstream task, to maintain the knowledge from the pre-training period, LoRA keeps the pre-trained parameters $\bm{W}_0 \in \mathbb{R}^{n_1 \times n_2}$ unmodified, and updates a zero-initialized delta weight matrix $\Delta \bm{W}$. The forward pass is:
\begin{equation}
    \bm{h} = \bm{W_0}\bm{x}+\Delta \bm{W}\bm{x}=\bm{W}_0\bm{x}+\bm{B}\bm{A}\bm{x}.
\label{eq:lora}
\end{equation}
To control the number of tunable parameters, as shown in Eq.\ref{eq:lora}, the delta weight matrix $\Delta W$ can be further decomposed into the product of two low-rank matrices, $\bm{A} \in \mathbb{R}^{r \times n_2}$ and $\bm{B} \in \mathbb{R}^{n_1\times r}$, where $r \ll \{n_1,n_2\}$.

\subsection{LoRA Dropout}
As dropout mechanisms have demonstrated great performance on control overfitting, in this work, for LoRA-based PEFT methods, we introduce a LoRA Dropout framework to improve the generalization ability when adapting to downstream tasks. Specifically, for a LoRA module described in Eq.\ref{eq:lora}, we randomly drop rows and columns from both tunable low-rank parameter matrices:
\begin{equation}
\begin{aligned}
    &\hat{\bm{A}} = \bm{A} \cdot \mathrm{diag}(\bm{m}_A), \bm{m}_A\sim \mathrm{Bern}(1-p);\\
    &\hat{\bm{B}} = \left(\bm{B}^\top \cdot \mathrm{diag}(\bm{m}_B)\right)^\top, \bm{m}_B\sim \mathrm{Bern}(1-p),
\end{aligned}
\label{eq:lora_dropout}
\end{equation}
where $\bm{m}_A \in \mathbb{R}^{n_2}$ and $\bm{m}_B \in \mathbb{R}^{n_1}$ are mask vectors drawn from the Bernoulli distribution, and $p$ denotes the probability that the parameters get dropped. Note that we conduct dropout on the input/output dimension of both matrices as applying dropout on the rank dimension would decrease the rank of LoRA, significantly impacting its expressive power. Additionally, performing dropout on the rank dimension will not increase the sparsity of the product of LoRA matrices, while theoretical evidence in the next section highlights the significance of sparsity in our framework.

With the dropout, the forward pass with dropout would be
\begin{equation}
\begin{aligned}
    \hat{\bm{h}} =\bm{W}_0\bm{x}+\hat{\bm{B}}\hat{\bm{A}}\bm{x}.
\end{aligned}
\label{eq:lora_dropout_forward}
\end{equation}
It should be noted that our dropout method is not only applicable to the original LoRA, but also equally suitable for LoRA-based variant methods, as long as they take the form of low-rank matrix decomposition. For example, AdaLoRA~\cite{zhang2023adaptive} conducts the decomposition through a quasi-SVD method,
\begin{equation}
\begin{aligned}
    \Delta \bm{W}=\bm{P}\bm{\Lambda}\bm{Q},
\end{aligned}
\label{eq:adalora}
\end{equation}
where $\bm{P} \in \mathbb{R}^{n_1 \times r}$ and $\bm{Q} \in \mathbb{R}^{r\times n_2}$ are left/right singular vectors, and $\bm{\Lambda} \in \mathbb{R}^{r \times r}$ is a diagonal matrix containing singular values. It's easy to adapt our LoRA Dropout to AdaLoRA through:
\begin{equation}
\begin{aligned}
    &\hat{\bm{P}} = \left(\bm{P}^\top \cdot \mathrm{diag}(\bm{d}_P)\right)^\top, \bm{d}_P\sim \mathrm{Bern}(1-p), \\
    &\hat{\bm{Q}} = \bm{Q} \cdot \mathrm{diag}(\bm{m}_Q), \bm{m}_Q\sim \mathrm{Bern}(1-p).
\end{aligned}
\label{eq:adalora_dropout}
\end{equation}
Dropout is not applied on the $\mathbf{\Lambda}$ matrix as it will also lead to rank shrinking and further influence the expressive power. Moreover, $\mathbf{\Lambda}$ will be adjusted by the AdaLoRA algorithm by filtering out minor compositions in practice, hence we conduct no further dropout on it.
After dropout, the delta weight matrix would be 
\begin{equation}
\begin{aligned}
    &\Delta \hat{\bm{W}}=\hat{\bm{P}}\bm{\Lambda}\hat{\bm{Q}}.
\end{aligned}
\label{eq:adalora_delta_weight}
\end{equation}
We provide a schematic diagram in Figure \ref{fig:lora_dropout} illustrating the integration of the proposed LoRA Dropout framework with both LoRA and AdaLoRA methods.

\begin{figure}[t]
\centering
  \includegraphics[width=2.8 in]{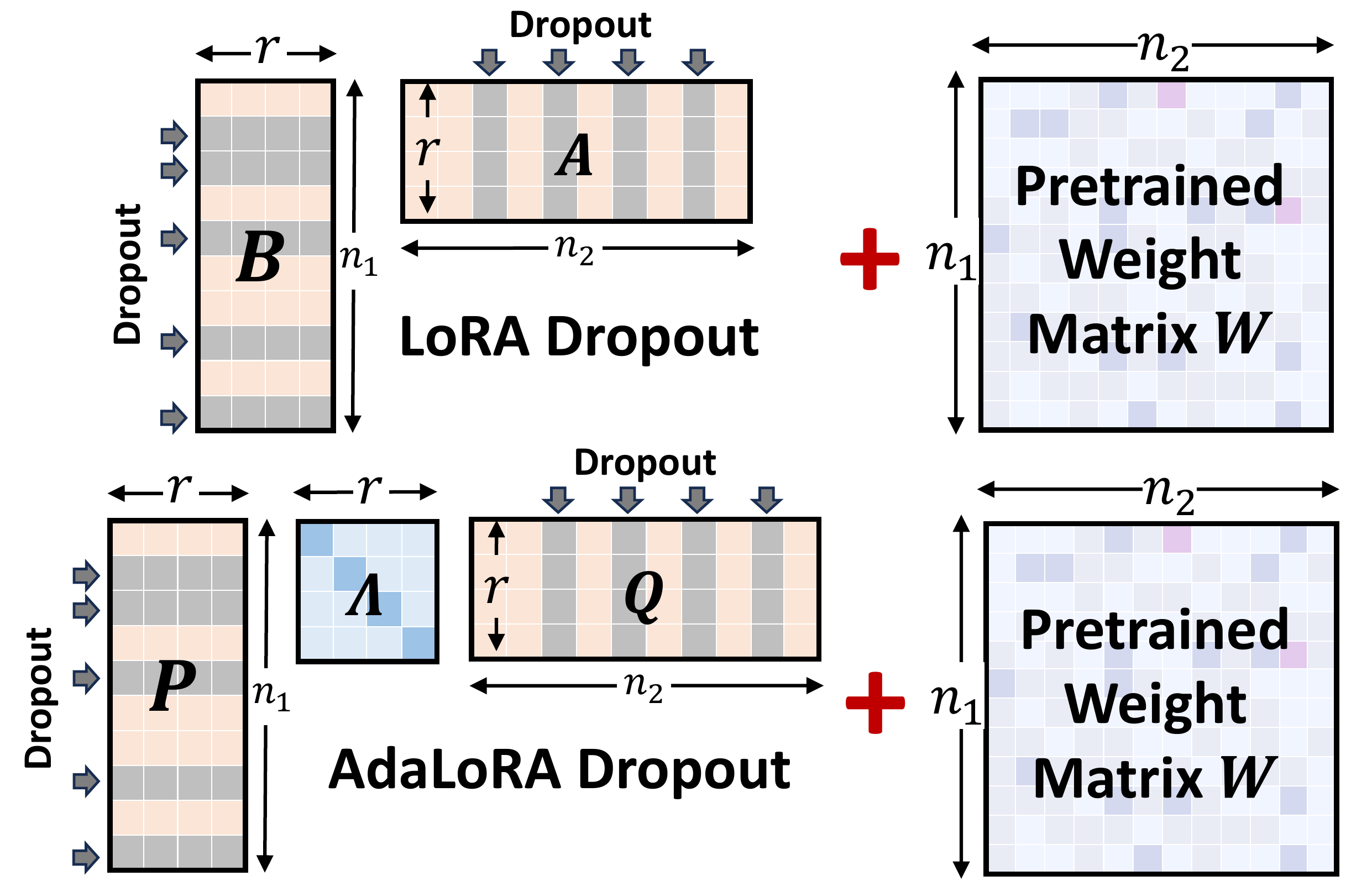}
  \caption{Our proposed dropout framework combined with both LoRA and AdaLoRA methods.}
  \label{fig:lora_dropout}
  \vspace{-0.12in}
\end{figure}

\subsection{Training Objective}
Let us denote $\bm{m}$ as the concatenation of all dropout vectors from LoRA module of a fine-tuning model, $\Delta\bm{\theta}(\bm{m})$ as the LoRA parameters after the dropout $\bm{m}$, and $\bm{\theta}^0$ as the original parameters of the pre-trained model. To obtain an effective model under various dropouts, we define the training objective as an average of multiple losses under $N$ different dropout instances on parameters,
\begin{equation}
\begin{aligned}
    \mathcal{L}(\bm{x})=\frac{1}{N}\sum_{r=1}^{N}\ell\left(\bm{x};\bm{\theta}^0+\Delta\bm{\theta}({\bm{m}_r})\right), \bm{m}_r \sim \mathrm{Bern}(1-p).
\end{aligned}
\label{eq:training_obj}
\end{equation}

\subsection{Test-time Ensemble}
To further enhance the model's performance during inference time, inspired by the MC dropout mechanism~\cite{gal2016dropout}, we propose a test-time ensemble method. Unlike the conventional dropout that is deactivated when testing, our ensemble strategy aggregates the outputs of models under different dropouts during inference time to get the final output, which can be viewed as sampling and aggregating models from a parameter distribution with a Monte Carlo method.
Specifically, let $\mathcal{M}\left(\bm{\theta}^0+\Delta\bm{\theta}({\bm{m}_r})\right)$ denote the model with LoRA parameter under dropout $\bm{m}_r$, then the output $\bm{o}$ of the ensemble model is 
\begin{equation}
\begin{aligned}
    & \bm{o}(\bm{x}) = \frac{1}{N}\sum_{r=1}^{N}\bm{o}_{r}(\bm{x}) = \frac{1}{N}\sum_{r=1}^{N} \mathcal{M}\left(\bm{x};\bm{\theta}^0+\Delta\bm{\theta}({\bm{m}_r})\right),
\end{aligned}
\label{eq:ensemble_method}
\end{equation}
where $N$ is the number of dropout instances. We provide a theoretical analysis in subsection \ref{section:ensemble_theory} for how this ensemble strategy can enhance performance during test-time.

To summarize, we provide an overall training and testing procedure for fine-tuning pre-trained models with the LoRA-based method and LoRA Dropout in Alg \ref{alg:training_procedure}.

\begin{algorithm}[t]
    \caption{The overall fine-tuning and testing procedure of a pre-trained model with LoRA Dropout.}
    \label{alg:training_procedure}
\begin{algorithmic}[1]
    \STATE {\bfseries Input:} total epoch number $T$, batch size $B$, dropout rate $p$, dropout instance number $N$.
    \STATE {\bfseries Training Phase:}
    \FOR{$epoch$ from 1 to $T$}
        \FOR{each iteration}
            \STATE randomly draw $B$ samples from the training set;
            \STATE $\mathcal{L}_{tr} \leftarrow 0$;
            \STATE draw $\bm{m}_r \sim \mathrm{Bern}(1-p)$, for $r$ in 1,...,$N$;
            \FOR{each sample $\bm{x}$ in batch}
                \STATE $\mathcal{L}_{tr} \leftarrow \mathcal{L}_{tr}+\mathcal{L}(\bm{x})$ by Eq.(\ref{eq:training_obj});
            \ENDFOR
            \STATE update tunable parameters with $\nabla\mathcal{L}_{tr}$.
        \ENDFOR
    \ENDFOR
    \STATE {\bfseries Test Phase:}
    \FOR{each sample $\bm{x}$ in test set}
        \STATE draw $\bm{m}_r \sim \mathrm{Bern}(1-p)$, for $r$ in 1,...,$N$;
        \STATE compute the ensemble output $\bm{o}(\bm{x})$ following Eq.(\ref{eq:ensemble_method}).
    \ENDFOR
\end{algorithmic}
\end{algorithm}

\section{Theoretical Results}
In this section, we present the theoretical results of our proposed methods from fine-tuning to inference perspectives. 
For the fine-tuning phase, we first model the fine-tuning with LoRA Dropout as an optimization problem under model sparsity regularization. Then we propose the generalization error bound under the sparsity regularization framework and reveal the theoretical mechanism behind the trade-off between underfitting and overfitting of LoRA Dropout fine-tuning. 
For the inference phase, we prove a tighter error bound with the test-time LoRA Dropout ensemble, demonstrating a better test-time generalizability of our framework.

\subsection{LoRA Dropout Fine-Tuning Through the Lens of Sparse Regularization}
Supposing a pre-trained model $\mathcal{M}^0$ parameterized by $\bm{\theta}^0 \in \mathbb{R}^d$, LoRA-like methods tune the model $\mathcal{M}^0$ with low-rank parameterization of the delta parameters $\Delta \bm{\theta}$. With our LoRA Dropout strategy, which samples random neurons on the input and output sides of LoRA matrices with a probability $p$ to mask them to zeros, the updated delta $\Delta \bm{\theta}$ enjoys a natural sparsity property that each entry in the product of the LoRA matrices will be zero with probability $1-(1-p)^2=2p-p^2$. 
Inspired by \cite{fu2023effectiveness}, we model this fine-tuning procedure as an optimization problem with model sparsity regularization. Let us denote $\bm{\theta} = \bm{\theta}^0 + \Delta \bm{\theta}$ as the fine-tuned model parameters, where $\Delta \bm{\theta}$ is realized by LoRA reparameterization with our LoRA Dropout method. Assume $\bm{d} \in \{ 0,1 \}^d$ as a dropout instance applied to the production of LoRA matrices (i.e., $\Delta \bm{\theta}$) sampled from a Bernoulli distribution, i.e., $\bm{d} \sim \mathrm{Bern}(2p-p^2)$, where $1$ denotes the corresponding entry is dropped to zero. The fine-tuning can be formulated as:
\begin{equation}
\begin{aligned}
    \min_{\Delta \bm{\theta}} \mathcal{L}(\bm{\theta}^0 + &  \Delta \bm{\theta}), \\
    \mathrm{s.t.} \ \mathbb{E}_{\bm{d} \sim \mathrm{Bern}(2p-p^2)} & ||\bm{d} \odot \Delta \bm{\theta}||_2^2 = 0.
\end{aligned}
\label{eq:sps_1}
\vspace{-0.01in}
\end{equation}
The condition denotes the sparsity of $\Delta \bm{\theta}$. By Lagrange duality, problem (\ref{eq:sps_1}) is equivalent to the following problem:
\begin{equation}
    \hat{\mathcal{L}} = \min_{\Delta \bm{\theta}} \max_{\lambda} \mathcal{L}(\bm{\theta}^0 +  \Delta \bm{\theta}) + \lambda \mathbb{E}_{\bm{d} \sim \mathrm{Bern}(2p-p^2)}  ||\bm{d} \odot \Delta \bm{\theta}||_2^2.
\label{eq:sps_reg}
\end{equation}
Hence, we formulate the regularized optimization problem:
\begin{equation}
    \mathcal{L}_{\lambda} = \min_{\Delta \bm{\theta}} \mathcal{L}(\bm{\theta}^0 +  \Delta \bm{\theta}) + \lambda \mathbb{E}_{\bm{d} \sim \mathrm{Bern}(2p-p^2)}  ||\bm{d} \odot \Delta \bm{\theta}||_2^2 \leq \hat{\mathcal{L}}.
\label{eq:sps_final}
\end{equation}
where $\lambda$ is an arbitrary hyperparameter. This optimization objective is upper bounded by $\hat{\mathcal{L}}$, which is equivalent to the optima of problem (\ref{eq:sps_1}).

\subsection{Generalization Analysis}\label{sec:generalization_analysis}
In this subsection, we introduce the stability analysis of a sparse-regularized algorithm to analyze the generalization error bound of LoRA Dropout fine-tuning through optimizing Eq. (\ref{eq:sps_final}). 
Stability has been a widely studied topic in machine learning \cite{bousquet2002stability, charles2018stability, kuzborskij2018data} and demonstrated as an important property for analyzing the generalization error bound of a random algorithm \cite{bousquet2002stability, elisseeff2005stability}.
Here we adopt one of the commonly used analytic mechanisms, the Pointwise Hypothesis Stability (PHS), which analyzes the perturbation of the optimal model after removing one of the training samples.
Following \cite{charles2018stability}, we denote the entire training dataset as $\mathbf{S} = \{ x_i \}_{i=1}^n$ and the dataset after removing a sample $x_i$ as $\mathbf{S}^i = \mathbf{S} - \{x_i\}$. We assume that $i\sim \mathrm{U}(n)$ that the removal is sampled from a uniform distribution. We also denote $\bm{\theta}_\ell (\mathbf{S})$ as the optimal model parameters w.r.t. loss function $\ell$ and dataset $\mathbf{S}$.

\begin{definition}[Pointwise Hypothesis Stability \cite{bousquet2002stability}]
    We say that a learning algorithm $\mathcal{M}$ parameterized by $\bm{\theta}$ w.r.t. a loss function $\ell$ has pointwise hypothesis stability $\beta$, if:
    \begin{equation}
        \mathbb{E}_{\mathbf{S}, i \sim \mathrm{U}(n)}\left|\ell\left(x_i; \bm{\theta}_\ell(\mathbf{S}^i) \right)-\ell\left( x_i; \bm{\theta}_\ell(\mathbf{S})\right)\right| \leq \beta,
        \vspace{-0.05in}
    \end{equation}
\end{definition}
where $\ell(x_i; \bm{\theta})$ denotes the sample loss of $x_i$ when the model parameter is $\bm{\theta}$. Here we present a PHS upper bound of our proposed LoRA Dropout framework.

\begin{proposition} [PHS Upper Bound of LoRA Dropout]
If the loss function $\mathcal{L}_\lambda$ of LoRA Dropout algorithm $\mathcal{M}$ is $\eta$-Lipschitz, and $\bm{\theta}_{\mathcal{L}_\lambda} (\mathbf{S}^i)$ is close to $\bm{\theta}_{\mathcal{L}_\lambda} (\mathbf{S})$, the Hessian matrix $\nabla^2 \mathcal{L}(\bm{\theta}_{\mathcal{L}_\lambda} (\mathbf{S}))$ at $\bm{\theta}_{\mathcal{L}_\lambda} (\mathbf{S})$ is positive-semidefinite with a singular value decomposition $U \operatorname{diag}(\Lambda) U^{-1}, \Lambda=\left\{\Lambda_1, \cdots, \Lambda_m\right\}$ and $\Lambda_{\min }=$ $\min \left\{\Lambda_1, \cdots, \Lambda_m\right\}$, then the LoRA Dropout algorithm optimizing $\mathcal{L}_{\lambda}$ on $\mathbf{S}$ has an upper bound of pointwise hypothesis stability of:
\begin{equation}
\begin{aligned}
    \mathbb{E}_{\mathbf{S}, i \sim \mathrm{U}(n)}\left|\mathcal{L}_{\lambda}\left(x_i; \bm{\theta}_{\mathcal{L}_\lambda} (\mathbf{S}^i)\right)-\mathcal{L}_{\lambda}\left(x_i; \bm{\theta}_{\mathcal{L}_\lambda} (\mathbf{S})\right)\right| \\
    \leq \frac{2 \eta^2}{\left(\Lambda_{\min }+2\lambda (2p-p^2)\right) n} .
\end{aligned}
\end{equation}
\label{prop:phs}
\vspace{-0.27in}
\end{proposition}

\begin{hproof}
    We first analyze the PHS upper bound of an arbitrary optimization algorithm equipped with $\ell_2$-regularizer in Lemma \ref{lem:apdx1}. Then we formulate our LoRA Dropout training objective as a similar problem with weighted $\ell_2$-regularizer and take it into Lemma \ref{lem:apdx1} to finish the proof. See Appendix \ref{apdx:proof_prop4.2} for detailed proofs.
    \vspace{-0.14in}
\end{hproof}

Moreover, existing works \cite{bousquet2002stability, elisseeff2005stability} have connected the stability and generalization error bound with the following lemma adopted from \citep[Theorem 11]{bousquet2002stability}. 
\begin{lemma}
\label{lemma:phs_errorbound}
 For any learning algorithm $\mathcal{M}$ having parameter $\bm{\theta}$ and bounded loss function $\ell$ satisfying $0 \leq |\ell(x) - \ell(x')| \leq C, \forall x,x'$. If $\mathcal{M}$ has a pointwise hypothesis stability $\beta$, with probability $1-\delta$, we have:
\begin{equation}
    R(\mathcal{M}, \mathbf{S}) \leq \hat{R}(\mathcal{M}, \mathbf{S})+\sqrt{\frac{C^2+12 C n \beta}{2 n \delta}},
\end{equation}
\vspace{-0.2in}
\end{lemma}
where $R(\mathcal{M}, \mathbf{S}) = \mathbb{E}_x \ell(x; \bm{\theta})$ and $\hat{R}(\mathcal{M}, \mathbf{S}) = \frac{1}{n} \sum_{i=1}^n \ell(x_i; \bm{\theta})$ denote the empirical risk and generalization risk of algorithm $\mathcal{M}$ running on dataset $\mathbf{S}$, repectively.
This indicates that better algorithm stability will reduce the complexity of the adaptation function class.
Therefore, invoking the PHS upper bound $\beta$ of LoRA Dropout algorithm in Proposition \ref{prop:phs} to Lemma \ref{lemma:phs_errorbound}, we have the following theorem that depicts the generalization error bound of LoRA Dropout framework:
\begin{theorem}[LoRA Dropout Generalization Error Bound]
\label{theo:generalization_error}
    Given a LoRA Dropout rate $p$ and strength of sparsity regularization $\lambda$, if LoRA Dropout Algorithm have a $\eta$-Lipschitz loss function $\mathcal{L}_\lambda$, and $\bm{\theta}_{\mathcal{L}_\lambda} (\mathbf{S}^i)$ is close to $\bm{\theta}_{\mathcal{L}_\lambda} (\mathbf{S})$, the Hessian matrix $\nabla^2 \mathcal{L}(\bm{\theta}_{\mathcal{L}_\lambda} (\mathbf{S}))$ at $\bm{\theta}_{\mathcal{L}_\lambda} (\mathbf{S})$ is positive-semidefinite with a singular value decomposition $U \operatorname{diag}(\Lambda) U^{-1}, \Lambda=\left\{\Lambda_1, \cdots, \Lambda_m\right\}$ and $\Lambda_{\min }=$ $\min \left\{\Lambda_1, \cdots, \Lambda_m\right\}$, then for some constant $C$, we have with probability $1-\delta$,
    \begin{equation}
        R(\mathcal{M}, \mathbf{S}) \leq \hat{R}(\mathcal{M}, \mathbf{S})+\sqrt{\frac{C^2+\frac{24C\eta^2}{\Lambda_{\min} + 2\lambda(2p-p^2)}}{2 n \delta}}.
    \end{equation}
    \vspace{-0.2in}
\end{theorem}
This theorem reveals the theoretical mechanism behind the trade-off between underfitting and overfitting of LoRA Dropout. The theorem shows that the complexity of adaptation function class (i.e., the gap between empirical and generalization risks) gets larger as the dropout rate gets smaller. Specifically, when applying traditional LoRA without dropout, the gap will be the largest, which depicts the high risk of overfitting with LoRA fine-tuning. However, when the dropout rate gets too large and tends to 1, it is equivalent to conducting no fine-tuning, which increases the empirical risk and makes the model underfit the training data. Hence, an appropriate dropout mechanism can theoretically balance a trade-off between the empirical risk minimization and the complexity of adaptation function classes, thereby enhancing the test-time performances as well as learning sufficiently from data.

\subsection{Error Bound of LoRA Dropout Ensemble Classifier}
\label{section:ensemble_theory}
For the inference phase, we aim to further control the error bound of LoRA Dropout by ensembling multiple test-time dropout instances. In this subsection, we provide a theoretical analysis of how the ensemble tightens the error bound.

\begin{table*}[t]
\caption{Results with DeBERTaV3-base on GLUE development set. The best-performing method is highlighted in \textbf{bold}. We report the matched accuracy for MNLI, Matthew’s correlation for CoLA, average correlation for STS-B, and accuracy for other tasks. Results are averaged over 5 runs using different random seeds.}
\label{tab:glue_results}
\begin{center}
\begin{small}
\begin{tabular}{l|c|ccccccccc}
\toprule
\multirow{2}{*}{\textbf{Method}}&\multirow{2}{*}{\textbf{\#Params}} &\textbf{MNLI} &\textbf{ SST-2}&\textbf{CoLA} &\textbf{QQP}&\textbf{QNLI}&\textbf{RTE}&\textbf{MRPC}&\textbf{STS-B}& \textbf{All} \\
 && M-Acc & Acc & Mcc& Acc &Acc &Acc &Acc &Corr &Avg.\\
\midrule
Full Fine-Tuning &184M   & 89.90 & 95.63& 69.19 & \textbf{92.40} & 94.03& 83.75 & 89.46 & 91.60& 88.25 \\ \midrule
BitFit &0.1M   & 89.37 & 94.84& 66.96 & 88.41 & 92.24& 78.70 & 87.75 & 91.35& 86.20 \\
H-Adapter &1.22M   & 90.13 & 95.53& 66.64 & 91.91 & 94.11& 84.48 & 89.95 & 91.48& 88.28 \\
P-Adapter &1.18M   & 90.33 & 95.61& 68.77 & 92.04 & 94.29& 85.20 & 89.46 & 91.54& 88.41 \\

\midrule
$\textup{LoRA}_{r=8}$ &1.33M   & 90.65 & 94.95& 69.82 & 91.99 & 93.87& 85.20 & 89.95 & 91.60& 88.50 \\
LoRA+Dropout&1.33M& \textbf{90.85}&	95.87 &71.32 &	92.22&	\textbf{94.56} & 88.09&\textbf{91.42} 	&92.00 & 89.54  \\
\midrule
AdaLoRA&1.27M   & 90.76 &96.10& 71.45 & \textbf{92.23} &94.55& 88.09 & 90.69&91.84& 89.46 \\
AdaLoRA+Dropout& 1.27M& 90.75& \textbf{96.22} & \textbf{72.04} &92.04&94.47 &\textbf{88.81}& 91.18&\textbf{92.07} & \textbf{89.70}  \\
\bottomrule
\end{tabular}
\end{small}
\end{center}
\vskip -0.15in
\end{table*}

During the fine-tuning phase, we optimize Eq.(\ref{eq:sps_final}) through accumulating gradient steps under different dropout instances. This fine-tuning procedure is essentially optimizing the generalization risks given the distribution $\mathcal{D}$ of model parameters $\bm{\theta}$, which is $\mathbb{E}_{\bm{\theta} \sim \mathcal{D}} \mathbb{E}_{(x,y)} \mathcal{L}_\lambda(\mathcal{M}(x; \bm{\theta}), y)$, where $\mathcal{M}(x; \bm{\theta})$ denotes the output of model $\mathcal{M}$ given the input $x$ parameterized by $\bm{\theta}$.
During the inference phase, with the test-time ensemble strategy, we are actually aggregating model outputs across the distribution $\mathcal{D}$ of parameter $\bm{\theta}$ to conduct final predictions, namely the ensemble classifier, which has an error of $\mathbb{E}_{(x,y)} \mathcal{L}_\lambda(\mathbb{E}_{\bm{\theta} \sim \mathcal{D}} \mathcal{M}(x; \bm{\theta}), y)$.
We present the following theorem that depicts a tighter generalization error bound with the ensemble classifier:

\begin{theorem}[Error Bound of Test-time LoRA Dropout Ensemble]
\label{theo:bayes_cls_error}
    If the loss function $\mathcal{L}_\lambda$ is convex w.r.t. the final activation $\bm{h}$ of model $\mathcal{M}$ before the output layer (e.g., softmax), then we have:
    \begin{equation}
    \begin{small}
    \begin{aligned}
        \mathbb{E}_{(x,y)} \mathcal{L}_\lambda(\mathbb{E}_{\bm{\theta} \sim \mathcal{D}} \mathcal{M}(x; \bm{\theta}), y) \leq \mathbb{E}_{\bm{\theta} \sim \mathcal{D}} \mathbb{E}_{(x,y)} \mathcal{L}_\lambda(\mathcal{M}(x; \bm{\theta}), y).
    \end{aligned}
    \end{small}
    \label{eq:bayes_cls}
    \end{equation}
    \vspace{-0.25in}
\end{theorem}
\begin{hproof}
    Taking expectation on parameter $\bm{\theta}$ is equivalent to taking expectation on the final hidden activation $\bm{h} \in \mathbb{R}^d$ from a certain distribution. Then simply apply Jensen inequality under the convex condition of $\bm{h} \in \mathbb{R}^d$ and the inequality holds. See Appendix \ref{apdx:proof_theo4.5} for detailed proofs.
    \vspace{-0.15in}
\end{hproof}
Moreover, the convexity holds for most cases in LLM training or fine-tuning scenarios, as we often take cross-entropy as the loss function and the softmax as the final output layer, and those functions are convex in the entire space $\mathbb{R}^d$. Hence, the inequality says that the generalization error of the ensemble classifier (i.e., LHS of (\ref{eq:bayes_cls})) is no greater than the training generalization error (i.e., RHS of (\ref{eq:bayes_cls})) for most LLM tuning scenarios, implying that the ensemble classifier with LoRA Dropout can further compress the error bound given by the LoRA Dropout fine-tuning and demonstrate better test-time generalizability. 

To summarize, Theorem \ref{theo:generalization_error} and \ref{theo:bayes_cls_error} together depict the full theoretical sketch of our practical framework. The fine-tuning phase applies LoRA Dropout to control the generalization error by balancing the trade-off between the empirical risk minimization and the complexity of adaptation, and the inference phase applies multiple dropout instances to accomplish an ensemble classifier with a tighter error bound and further enhances the test-time generalizability.

\section{Experiments}

In this section, we conduct a series of experiments to validate the effectiveness of our proposed LoRA Dropout framework. We incorporate LoRA Dropout into LoRA-series works, LoRA and AdaLoRA, and compare them with original models and other baselines on various tasks. Then we test whether our LoRA Dropout can further improve model calibration. Lastly, we conduct ablation studies to verify the effectiveness of each component in our method.


\textbf{Baselines  }
We compared the our method with following state-of-the-art PEFT methods. 
\begin{itemize}[leftmargin=*]
\vspace{-0.05in}
    \item \textbf{Full fine-tuning}~\cite{howard2018universal}-  All pre-trained parameters within the model get trained and optimized.
    \item \textbf{BitFit}~\cite{zaken2022bitfit}- Only the bias vectors from the model parameters get fine-tuned.
    \item \textbf{H-Adapter}~\cite{houlsby2019parameter}- The adapters are inserted between the MLP and the self-attention modules.
    \item \textbf{P-Adapter}~\cite{pfeiffer2020adapterfusion}- Adapter layers are applied only after the MLP or the LayerNorm layer.
    \item \textbf{LoRA}~\cite{zhang2023adaptive}- LoRA decomposes the learnable delta parameter matrix into two low-rank matrices to improve parameter efficiency.
    \item \textbf{AdaLoRA}~\cite{zhang2023adaptive}- AdaLoRA introduces an adaptive parameter budget by gradually pruning the rank of LoRA based on sensitivity-based importance scores.
    
\end{itemize}

When comparing with baseline models, we keep tunable parameter budgets for all methods aligned. We pick the hyperparameters settings following our base models, i.e. LoRA and AdaLoRA, and tune the hyperparameters that are exclusive to our model. More detailed experiment settings can be viewed in Appendix \ref{apdx:experimental_details} and \ref{apdx:dataset_details}.

\begin{table*}[ht]
  \centering 
  \caption{Results with DeBERTaV3-base on SQuAD v1.1 and SQuADv2.0. \#Params is the ratio of trainable parameters. We report EM (Exact Match) and F1 for each model. The best results in each setting are shown in \textbf{bold} and the second best methods are  \underline{underlined}.}
  \vskip 0.08in
  \begin{small}
    \begin{tabular}{l|cccccc|cccccc}
    \toprule
      & \multicolumn{6}{c|}{\textbf{SQUAD v1.1}} & \multicolumn{6}{c}{\textbf{SQUAD v2.0}} \\

    \midrule
             \textbf{\#Params}        & \multicolumn{2}{c}{0.16\%}     & \multicolumn{2}{c}{0.32\%}        & \multicolumn{2}{c|}{0.65\%}    & \multicolumn{2}{c}{0.16\%}       & \multicolumn{2}{c}{0.32\%}        & \multicolumn{2}{c}{0.65\%}         \\
    \midrule
            \textbf{Metric}       & EM&F1         & EM&F1        & EM&F1   & EM&F1       & EM&F1     & EM&F1      \\
    \midrule

    HAdapter   	&85.3 &92.1&	86.1&92.7	&	86.7&92.9&	84.3&87.3&	84.9&87.9 &85.4&88.3 \\
    PAdapter    &85.9&92.5	&86.2&92.8&	86.6&93.0	&	84.5&87.6&	84.9&87.8&	84.5&87.5 \\
    \midrule
    LoRA          &86.6&92.9	&86.7&93.1&	86.7&93.1	&	83.6&86.7&	84.5&87.4&	85.0&88.0 \\
    LoRA+Dropout  &\textbf{88.2}&\underline{93.8}&	\textbf{88.7}&\underline{94.1}&	\textbf{88.7}&\underline{94.2}		&85.4&88.4&	\textbf{86.0}&\underline{88.8}	&\underline{86.1}&\underline{88.9}   \\
    \midrule
        AdaLoRA &87.5&93.6&	87.5&93.7&	\underline{87.6}&93.7		&\underline{85.7}&\textbf{88.8}&	85.5&88.6	&86.0&\underline{88.9}     \\  
    AdaLoRA+Dropout  &\underline{88.1}&\textbf{93.9}&	\underline{88.5}&\textbf{94.2}&	\textbf{88.7}&\textbf{94.3}		&\textbf{85.8}&\underline{88.6}&	\underline{85.9}&\textbf{88.9}	&\textbf{86.3}&\textbf{89.1}   \\ 
    \bottomrule
    \end{tabular}
    \label{tab:squad_results}
\end{small}

\vspace{-10 pt}
\end{table*}

\subsection{Natural Language Understanding}\label{exp:main results}

\textbf{Settings  }
Following previous work~\cite{zhang2023adaptive}, we use the General Language Understanding Evaluation (GLUE) benchmark~\cite{wang2018glue} for evaluation. Our experiments contain eight different tasks from the GLUE benchmark. All models are fine-tuned on the DeBERTaV3-base~\cite{he2021debertav3} pre-trained model.

\textbf{Results  }
The results of different models on the GLUE benchmark are shown in Table \ref{tab:glue_results}. From the results, we could find that LoRA-series models with our LoRA Dropout framework consistently outperform other baselines, and achieve the best performance on the overall result among eight NLU tasks. Moreover, when compared with the original LoRA and AdaLoRA models, models with our dropout method always achieve superior performance, indicating that the proposed LoRA Dropout framework could help LoRA-based models control overfitting and improve generalization ability on downstream tasks.

\subsection{Question Answering}\label{exp:question answering}

\textbf{Settings  } We conduct the question answering task on the SQuAD (Stanford Question Answering Dataset) benchmark with two versions, SQuAD v1.1~\cite{rajpurkar2016squad} and SQuAD v2.0~\cite{rajpurkar2018know}. The DeBERTaV3-base model is used as base pre-trained model. We report the Exact Match accuracy and F1 score for each model.

\textbf{Results  }
The results of different models on the SQuAD benchmarks are shown in Table \ref{tab:squad_results}. The results further validate the conclusions that we obtained from Table \ref{tab:glue_results}. The LoRA Dropout method helps the base model (i.e., LoRA and AdaLoRA) to achieve better performances on both SQuAD benchmarks. Moreover, by varying the budget of trainable parameters (i.e., the hidden dimension of adapters and the rank of LoRA module), we could find that our method has consistently superior performance under various parameter budgets, revealing its effectiveness.

\subsection{Instruction Tuning}
\textbf{Settings  }
We test the models' instruction-following ability by conducting instruction tuning. Specifically, we choose LLaMA2-7B~\cite{touvron2023llama} as the pre-trained base model and fine-tune on the Alpaca-clean dataset\footnote{https://huggingface.co/datasets/yahma/alpaca-cleaned}~\cite{alpaca} with both original LoRA method and LoRA with LoRA Dropout. MMLU benchmark~\cite{hendrycks2020measuring} is employed to evaluate each model.

\textbf{Results  } We report the accuracy on the MMLU benchmark under 0-shot setting in Table \ref{tab:llama2_results}. From the results, we could find that the fine-tuned LoRA model performs worse than the original pre-trained LLaMA2-7B model, which indicates that LoRA overfits on the fine-tuning dataset and cannot generalize well on the evaluation benchmark due to the distribution disparity between fine-tuning set and evaluation set. However, LoRA with our dropout framework achieves much better accuracy than both models, which again demonstrates our method's ability to control overfitting.

\begin{table}[t]
\caption{Results of instruction tuning on LLaMA2-7B. We report Accuracy(\%) on MMLU 0-shot setting. Best results are in \textbf{bold}.}
\label{tab:llama2_results}
\begin{center}
\begin{small}
\begin{tabular}{l|ccccc}
\toprule
\multirow{2}{*}{\textbf{Method}} & \multicolumn{5}{c}{\textbf{MMLU (0-shot)}}                                         \\
                        & STEM           & Social         & Hum.           & Other          & Avg.           \\ \midrule
LLaMA2-7B      & 33.31          & 46.78          & 38.76          & 45.04          & 40.79          \\
$\textup{LoRA}_{r=16}$          & \textbf{34.40} & 45.15          & 38.19          & 45.60          & 40.61          \\ \midrule
LoRA+Dropout   & 34.07          & \textbf{48.71} & \textbf{40.52} & \textbf{47.18} & \textbf{42.47} \\ \bottomrule
\end{tabular}
\end{small}
\end{center}
\vspace{-10 pt}
\end{table}

\subsection{Confidence Calibration}\label{exp:confidence calibration}
\textbf{Settings  } As large-scale pre-trained models often exhibit overconfidence~\cite{jiang2021can,xiao2022uncertainty,he2023preserving,tian2023just}, we evaluate the confidence calibration~\cite{ guo2017calibration} of each model, which serves as an effective analytical method for evaluating model reliability~\cite{ zhu2023calibration}. Specifically, we employ the Expected Calibration Error (ECE) for measuring the calibration performance, and assess the confidence calibration of different fine-tuned models on a few tasks from the GLUE benchmark based on the DeBERTaV3-base model.

\textbf{Results  }
We report the ECE results of each model in Table \ref{tab:ECE_results} when it reaches the best performance on the development set, and also provide the ECE curves of different models in Figure \ref{fig:ece_curve} when fine-tuned on the RTE task. From the results we could find that LoRA Dropout could consistently reduce the ECE compared with its base model, leading to better-calibrated models. One possible explanation is that LoRA Dropout can be viewed as a variant of the MC dropout from a Bayes perspective. By randomly dropping parameters during training, we are estimating the posterior weight distributions with a given downstream task, making the model a kind of Bayes neural network, which is known to achieve good calibration~\cite{kristiadi2020being}.

\begin{table}[t]
\caption{The Expected Calibration Error (ECE $\downarrow$) of different models fine-tuned on tasks from GLUE benchmark.}

\label{tab:ECE_results}
\vskip 0.15in
\begin{center}
\begin{small}
\begin{tabular}{l|ccc}
\toprule
\textbf{Method} &\textbf{ SST-2} &\textbf{RTE}&\textbf{MRPC} \\
\midrule
$\textup{LoRA}_{r=8}$& 3.61 & 14.45 & 11.00  \\
LoRA+Dropout& 3.07 & 9.88 & 8.56   \\
\midrule
AdaLoRA&3.09 &12.12&  8.62\\
AdaLoRA+Dropout& 2.59&11.15 &5.07\\
\bottomrule
\end{tabular}
\end{small}
\end{center}
\vskip -0.1in
\end{table}

\begin{figure}[t]
\centering
  \includegraphics[width=3.0 in]{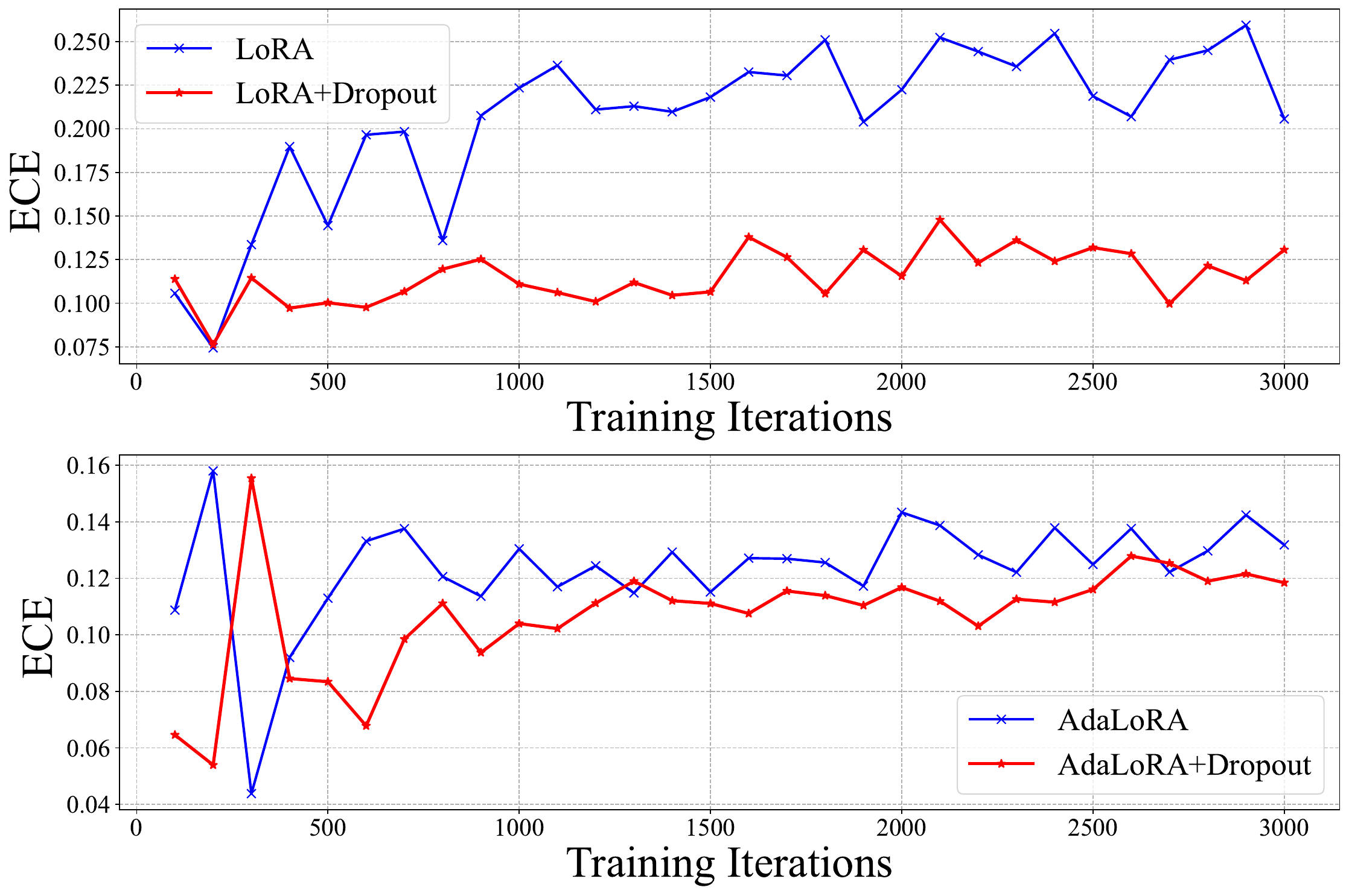}
  \vspace{-10 pt}
  \caption{The Expected Calibration Error (ECE $\downarrow$) during the fine-tuning process of RTE task.}
  \vspace{-0.1 in}
  \label{fig:ece_curve}
\end{figure}

\subsection{Ablation Studies and Sensitivity Analysis}
\textbf{Effect of Train/Test Dropout  }
We conduct experiments on the effects of dropout during training and testing on the MRPC and CoLA dataset, and the results are shown in Figure \ref{fig:ablation}. We compared our method (LoRA with our LoRA Dropout framework, denoted as \textit{Drop}) with the following variants:
$\mathrm{\textit{Drop}}_{train}$ denotes training with dropout and testing without dropout ensemble. $\mathrm{\textit{Drop}}_{test}$ denotes training without dropout and testing with dropout ensemble of dropout instances. And \textit{NoDrop} denotes model without LoRA Dropout. 
We could find the importance of conducting dropout during fine-tuning from the bad performance of $\mathrm{\textit{Drop}}_{test}$. It performs worse than the vanilla \textit{NoDrop} model since testing dropout may break some hidden semantic structure of parameters learned from training.
We can also verify the effectiveness of test-time ensemble strategy from the decrease bewteen \textit{Drop} and $\mathrm{\textit{Drop}}_{train}$, which aligns with our theoretical derivation that the ensemble strategy would further compress the error bound. 

\textbf{Effect of Dropout Rate  }
We conduct experiments on the effects of dropout rate $p$ on the MRPC dataset, and show the results in Figure \ref{fig:sensitive}. As the dropout rate increases, the performance first improves and then drops. This aligns with our theoretical derivation in Section \ref{sec:generalization_analysis} that a proper dropout rate would help balance the empirical risk minimization and complexity of the adaptation function. A small dropout rate might fail to introduce sufficient sparsity and lead to overfitting, while an excessively large dropout rate would result in too few trainable parameters, making the adapter lose its expressive power.

\textbf{Effect of Number of Sampled Dropout Instances  } 
We conduct experiments on the effects of dropout instance number $N$ on the MRPC dataset, and the results are shown in Figure \ref{fig:sensitive}. From the results, we can find the model performance improves as the sample number increases. This is reasonable since with a larger sample number, more dropout instances can be introduced during training and more models get aggregated during the test-time ensemble, leading to more accurate estimations of the outputs over the parameter distribution. However, a larger sample number will also lead to higher training and inference costs, thus picking an appropriate $N$ is necessary for a better balance between accuracy and computational cost.

\begin{figure}[t]
\centering
  \includegraphics[width=3.2 in]{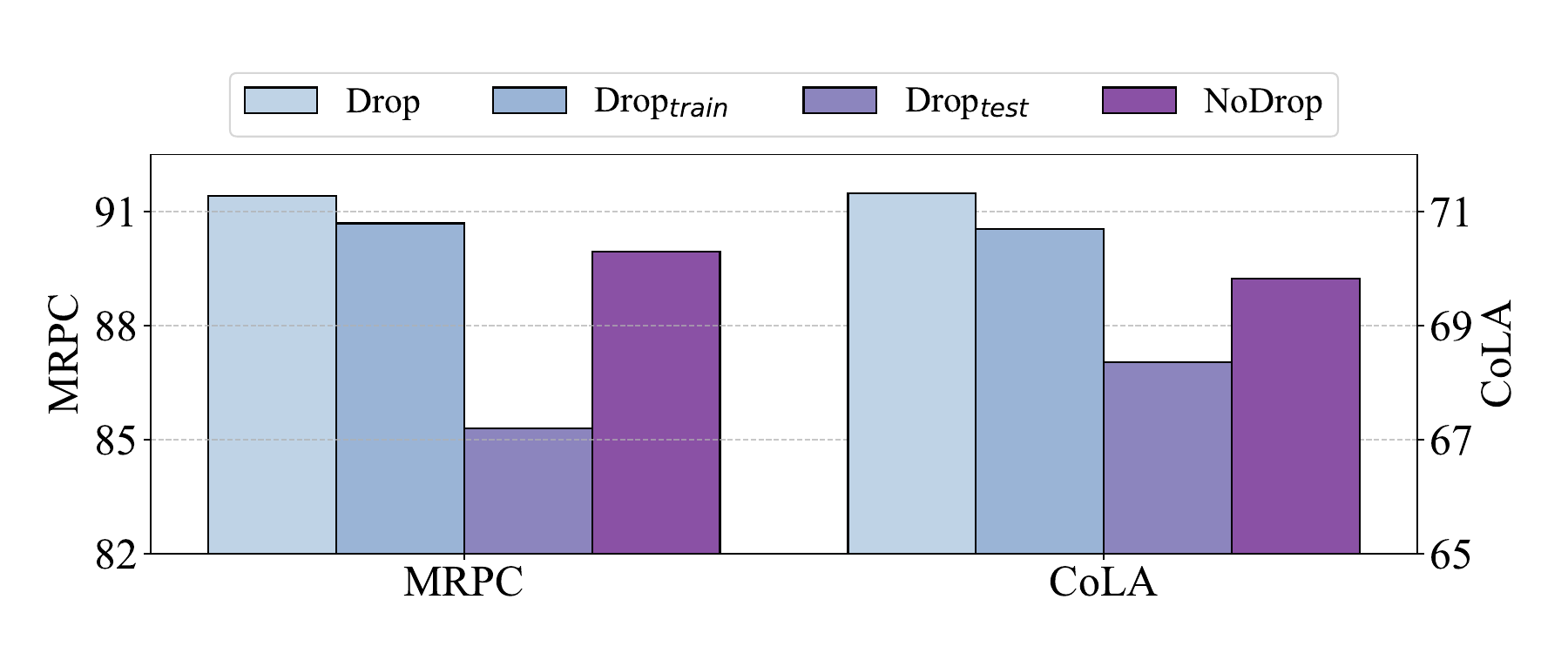}
  \vspace{-10 pt}
  \caption{Ablation studies on the dropout strategy.}
  \label{fig:ablation}
\end{figure}

\begin{figure}[t]
\centering
  \includegraphics[width=3.2in]{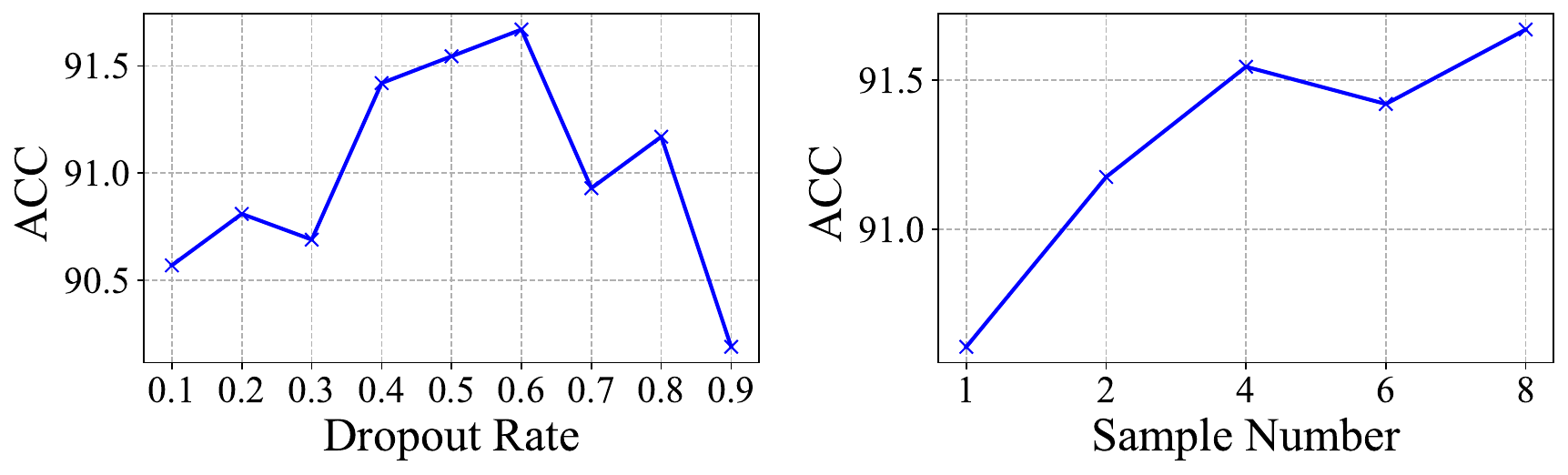}
  \vspace{-10 pt}
  \caption{Sensitive analysis on dropout rate and sample number.}
  \label{fig:sensitive}
  \vspace{-0.1 in}
\end{figure}

\section{Conclusions and Future Work}
To control the overfitting risk when fine-tuning on downstream tasks, in this paper, we propose a theoretically grounded LoRA Dropout framework designed for LoRA-based PEFT methods. Theoretical analyses from the perspective of sparse show that sparsity introduced by LoRA Dropout helps tighten the between empirical and generalization risks and thereby control overfitting. A test-ensemble strategy is proposed based on LoRA Dropout and theoretically shown to further compress the error bound. We conduct experiments on various tasks and PLMs, and the results demonstrate the effectiveness of our method on improving model accuracy and calibration.

Despite the promising results, we still want to point out the limitations. Though LoRA Dropout introduces no additional tunable parameters compared to LoRA, sampling multiple dropout instances during training and testing does introduces considerable time overhead. For future work, we aim to design a parallel computing framework for LoRA Dropout, expecting to improve in both performance and efficiency. 



\bibliography{example_paper}
\bibliographystyle{icml2024}

\newpage
\appendix
\onecolumn

\setcounter{table}{0} 
\setcounter{figure}{0} 
\setcounter{equation}{0} 

\section{Proofs of Theoretical Results}
\numberwithin{equation}{section}
\numberwithin{table}{section}
\subsection{Proof of Proposition \ref{prop:phs}} \label{apdx:proof_prop4.2}
We first prove a Lemma that describes the pointwise hypothesis stability of an optimization problem with $\ell_2$-regularizer, which provides an upper bound related to a constant describing the specific shape around the local optima. The symbols follow those in the main text. We first denote $\mathcal{L}(\bm{\theta}) = \frac{1}{n} \sum_i \mathcal{L}(x_i; \bm{\theta})$.

\begin{lemma}
    Consider the learning algorithm $\mathcal{M}$ optimizing the following loss function:
    \begin{equation*}
        \min_{\theta}\mathcal{L}_\lambda(\bm{\theta}) := \min_{\bm{\theta}} \mathcal{L}(\bm{\theta}) + \lambda ||\bm{\theta} - \bm{\theta}^0||_2^2.
    \end{equation*}If the loss function $\mathcal{L}$ is $\eta$-Lipschitz, and $\bm{\theta}_{\mathcal{L}_\lambda}(\mathbf{S})$ is close to $\bm{\theta}_{\mathcal{L}_\lambda}(\mathbf{S}^i)$. The Hessian matrix $\nabla^2 \mathcal{L}(\bm{\theta}_{\mathcal{L}_\lambda} (\mathbf{S}))$ at $\bm{\theta}_{\mathcal{L}_\lambda} (\mathbf{S})$ is positive-semidefinite with a singular value decomposition $U \operatorname{diag}(\Lambda) U^{-1}, \Lambda=\left\{\Lambda_1, \cdots, \Lambda_m\right\}$ and $\Lambda_{\min }=$ $\min \left\{\Lambda_1, \cdots, \Lambda_m\right\}$, Then $\mathcal{M}$ has pointwise hypothesis stability $\beta = \frac{2 \eta^2}{\left(\Lambda_{\min }+2\lambda \right) n}$, which is:
    \begin{equation*}
        \mathbb{E}_{\mathbf{S}, i \sim \mathrm{U}(n)}\left|\mathcal{L}_{\lambda}\left(x_i; \bm{\theta}_{\mathcal{L}_\lambda} (\mathbf{S}^i)\right)-\mathcal{L}_{\lambda}\left(x_i; \bm{\theta}_{\mathcal{L}_\lambda} (\mathbf{S})\right)\right| \\
    \leq \frac{2 \eta^2}{\left(\Lambda_{\min }+2\lambda \right) n} .
    \end{equation*}
    \label{lem:apdx1}
\end{lemma}

\begin{proof}
    For simplicity, we denote $\bm{\theta}_{\mathcal{L}_\lambda} (\mathbf{S})$ as $\hat{\bm{\theta}}$, and $\Delta \hat{\bm{\theta}} := \hat{\bm{\theta}} - \bm{\theta}$. Consider the second-order Taylor expansion of $\mathcal{L}_\lambda$ at local optima $\hat{\bm{\theta}}$, we have $\nabla_{\mathcal{L}_\lambda}(\hat{\bm{\theta}})=0$. For $\forall v$ close to $\hat{\bm{\theta}}$, we have:
    \begin{equation*}
        \begin{aligned}
        \mathcal{L}_\lambda(v) & = \mathcal{L}_\lambda(\hat{\bm{\theta}}) + (v- \hat{\bm{\theta}})^\top \nabla_{\mathcal{L}_\lambda}(\hat{\bm{\theta}}) + \frac{1}{2} (v- \hat{\bm{\theta}})^\top \nabla^2{\mathcal{L}_\lambda}(\hat{\bm{\theta}})(v- \hat{\bm{\theta}}) \\
         & = \mathcal{L}_\lambda(\hat{\bm{\theta}}) + \frac{1}{2} (v- \hat{\bm{\theta}})^\top \nabla^2{\mathcal{L}_\lambda}(\hat{\bm{\theta}})(v- \hat{\bm{\theta}})
        \end{aligned}
    \end{equation*}
    Then, we have:
    \begin{equation}
        \begin{aligned}
        \mathcal{L}_\lambda(v) - \mathcal{L}_\lambda(\hat{\bm{\theta}}) & = \frac{1}{2} (v- \hat{\bm{\theta}})^\top \nabla^2{\mathcal{L}_\lambda}(\hat{\bm{\theta}})(v- \hat{\bm{\theta}}) \\
        & = \frac{1}{2} (v- \hat{\bm{\theta}})^\top \nabla^2_{\hat{\bm{\theta}}}(\mathcal{L}(\hat{\bm{\theta}}) + \lambda ||\hat{\bm{\theta}} - \bm{\theta}^0||_2^2)(v- \hat{\bm{\theta}}) \\
        & = \frac{1}{2} (v- \hat{\bm{\theta}})^\top (\nabla^2\mathcal{L}(\hat{\bm{\theta}}) + 2\lambda I) (v- \hat{\bm{\theta}}) \\
        & = \frac{1}{2} (v- \hat{\bm{\theta}})^\top ( U \operatorname{diag}(\Lambda) U^{-1} + 2\lambda I) (v- \hat{\bm{\theta}}) \\
        & = \frac{1}{2} (v- \hat{\bm{\theta}})^\top ( U (\operatorname{diag}(\Lambda) + 2\lambda I) U^{-1}) (v- \hat{\bm{\theta}}) \\
        & = \frac{1}{2} (v- \hat{\bm{\theta}})^\top ( U \operatorname{diag}(\sqrt{\Lambda_1 +2\lambda}, \cdots, \sqrt{\Lambda_d +2\lambda}) U^{-1} U \operatorname{diag}(\sqrt{\Lambda_1 +2\lambda}, \cdots, \sqrt{\Lambda_d +2\lambda}) U^{-1}) (v- \hat{\bm{\theta}}) \\
        & = \frac{1}{2} ||(U \operatorname{diag}(\sqrt{\Lambda_1 +2\lambda}, \cdots, \sqrt{\Lambda_d +2\lambda} U^{-1}) (v- \hat{\bm{\theta}})||_2^2 \\
        & \geq \frac{1}{2}(\Lambda_{\min}+2\lambda) ||v- \hat{\bm{\theta}}||_2^2.
        \end{aligned}
        \label{eq:pf1}
    \end{equation}
    This inequality holds for the orthogonality of $U$ that it does not change the magnitude of vector $v- \hat{\bm{\theta}}$, and the magnitude is at least scaled with the minimum singular value $\Lambda_{\min}$.
    Then, by the definition of $\mathcal{L}_\lambda (\theta)$, for $\forall u, v$ close to $\hat{\bm{\theta}}$, we have:
    \begin{equation}
        \begin{aligned}
            \mathcal{L}_\lambda(u) - \mathcal{L}_\lambda(v) & = \left(\frac{1}{n} \sum_k{\mathcal{L}(x_k; u)} + \lambda ||u-\bm{\theta}^0||_2^2\right) - \left(\frac{1}{n} \sum_k{\mathcal{L}(x_k; v)} + \lambda ||v-\bm{\theta}^0||_2^2\right) \\
            & = \left(\frac{1}{n} \sum_{k\neq i}{\mathcal{L}(x_k; u)} + \lambda ||u-\bm{\theta}^0||_2^2\right) - \left(\frac{1}{n} \sum_{k\neq i}{\mathcal{L}(x_k; v)} + \lambda ||v-\bm{\theta}^0||_2^2\right) + \frac{\mathcal{L}(x_i;u) - \mathcal{L}(x_i;v)}{n}\\
            & = (1-\frac{1}{n})\left(\frac{1}{n-1} \sum_{k\neq i}{\mathcal{L}(x_k; u)} + \lambda ||u-\bm{\theta}^0||_2^2\right) - (1-\frac{1}{n})\left(\frac{1}{n-1} \sum_{k\neq i}{\mathcal{L}(x_k; v)} + \lambda ||v-\bm{\theta}^0||_2^2\right) + \\
            & \quad \frac{\lambda \left(||u-\bm{\theta}^0||_2^2 - ||v-\bm{\theta}^0||_2^2\right)}{n}  + \frac{\mathcal{L}(x_i;u) - \mathcal{L}(x_i;v)}{n}\\
            & = \underbrace{(1-\frac{1}{n})\left[\left(\frac{1}{n-1} \sum_{k\neq i}{\mathcal{L}(x_k; u)} + \lambda ||u-\bm{\theta}^0||_2^2\right) - \left(\frac{1}{n-1} \sum_{k\neq i}{\mathcal{L}(x_k; v)} + \lambda ||v-\bm{\theta}^0||_2^2\right)\right]}_{(*)} + \\ 
            & \quad \frac{\mathcal{L}_\lambda(x_i;u) - \mathcal{L}_\lambda(x_i;v)}{n}.
        \end{aligned}
        \label{eq:pf2}
    \end{equation}
    Taking $u = \bm{\theta}_{\mathcal{L}_\lambda}(\mathbf{S}^i)$ and $v = \bm{\theta}_{\mathcal{L}_\lambda}(\mathbf{S})$. As $u$ minimizes the empirical loss of removing $x_i$ out, hence the $(*)$ item in Eq.(\ref{eq:pf2}) is smaller than 0. Then, we have:
    \begin{equation*}
        \mathcal{L}_\lambda(\bm{\theta}_{\mathcal{L}_\lambda}(\mathbf{S}^i)) - \mathcal{L}_\lambda(\bm{\theta}_{\mathcal{L}_\lambda}(\mathbf{S})) \leq \frac{\mathcal{L}_\lambda(x_i;\bm{\theta}_{\mathcal{L}_\lambda}(\mathbf{S}^i)) - \mathcal{L}_\lambda(x_i;\bm{\theta}_{\mathcal{L}_\lambda}(\mathbf{S}^i))}{n}
    \end{equation*}
    Considering inequality of (\ref{eq:pf1}), we have:
    \begin{equation}
        \frac{1}{2}(\Lambda_{\min} + 2\lambda)||\bm{\theta}_{\mathcal{L}_\lambda}(\mathbf{S}^i) - \bm{\theta}_{\mathcal{L}_\lambda}(\mathbf{S})||_2^2 \leq \frac{\mathcal{L}_\lambda(x_i;\bm{\theta}_{\mathcal{L}_\lambda}(\mathbf{S}^i)) - \mathcal{L}_\lambda(x_i;\bm{\theta}_{\mathcal{L}_\lambda}(\mathbf{S}))}{n}
    \label{eq:pf3}
    \end{equation}
    As the loss function $\mathcal{L}_\lambda$ is $\eta$-Lipschitz, thus we have:
    \begin{equation}
        |\mathcal{L}_\lambda(x_i;\bm{\theta}_{\mathcal{L}_\lambda}(\mathbf{S}^i)) -\mathcal{L}_\lambda(x_i;\bm{\theta}_{\mathcal{L}_\lambda}(\mathbf{S}))| \leq \eta || \bm{\theta}_{\mathcal{L}_\lambda}(\mathbf{S}^i) - \bm{\theta}_{\mathcal{L}_\lambda}(\mathbf{S})||.
    \label{eq:pf4}
    \end{equation}
    Taking (\ref{eq:pf4}) into (\ref{eq:pf3}), we have:
    \begin{equation}
        \begin{aligned}
        \frac{1}{2}(\Lambda_{\min} + 2\lambda)||\bm{\theta}_{\mathcal{L}_\lambda}(\mathbf{S}^i) - \bm{\theta}_{\mathcal{L}_\lambda}(\mathbf{S})||_2^2 & \leq \frac{\eta || \bm{\theta}_{\mathcal{L}_\lambda}(\mathbf{S}^i) - \bm{\theta}_{\mathcal{L}_\lambda}(\mathbf{S})||}{n} \\ 
        \Rightarrow || \bm{\theta}_{\mathcal{L}_\lambda}(\mathbf{S}^i) - \bm{\theta}_{\mathcal{L}_\lambda}(\mathbf{S})|| & \leq \frac{2\eta}{(\Lambda_{\min} + 2\lambda)n}.
        \end{aligned}
        \label{eq:pf5}
    \end{equation}
    Plugging (\ref{eq:pf5}) back to (\ref{eq:pf4}):
    \begin{equation}
        |\mathcal{L}_\lambda(x_i;\bm{\theta}_{\mathcal{L}_\lambda}(\mathbf{S}^i)) -\mathcal{L}_\lambda(x_i;\bm{\theta}_{\mathcal{L}_\lambda}(\mathbf{S}))| \leq \frac{2\eta^2}{(\Lambda_{\min} + 2\lambda)n}.
    \end{equation}
    As this holds for any $i$ and $\mathbf{S}$, hence we have:
    \begin{equation*}
        \mathbb{E}_{\mathbf{S}, i \sim \mathrm{U}(n)}\left|\mathcal{L}_{\lambda}\left(x_i; \bm{\theta}_{\mathcal{L}_\lambda} (\mathbf{S}^i)\right)-\mathcal{L}_{\lambda}\left(x_i; \bm{\theta}_{\mathcal{L}_\lambda} (\mathbf{S})\right)\right| \\
    \leq \frac{2 \eta^2}{\left(\Lambda_{\min }+2\lambda \right) n} .
    \end{equation*}
\end{proof}

Based on this Lemma, we aim to analyze our optimization objective of Eq.(\ref{eq:sps_final}) and prove Proposition \ref{prop:phs} as follows.

\begin{proposition} [PHS Upper Bound of LoRA Dropout]
If the loss function $\mathcal{L}_\lambda$ of LoRA Dropout algorithm $\mathcal{M}$ is $\eta$-Lipschitz, and $\bm{\theta}_{\mathcal{L}_\lambda} (\mathbf{S}^i)$ is close to $\bm{\theta}_{\mathcal{L}_\lambda} (\mathbf{S})$, the Hessian matrix $\nabla^2 \mathcal{L}(\bm{\theta}_{\mathcal{L}_\lambda} (\mathbf{S}))$ at $\bm{\theta}_{\mathcal{L}_\lambda} (\mathbf{S})$ is positive-semidefinite with a singular value decomposition $U \operatorname{diag}(\Lambda) U^{-1}, \Lambda=\left\{\Lambda_1, \cdots, \Lambda_m\right\}$ and $\Lambda_{\min }=$ $\min \left\{\Lambda_1, \cdots, \Lambda_m\right\}$, then the LoRA Dropout algorithm optimizing $\mathcal{L}_{\lambda}$ on $\mathbf{S}$ has an upper bound of pointwise hypothesis stability of:
\begin{equation*}
    \mathbb{E}_{\mathbf{S}, i \sim \mathrm{U}(n)}\left|\mathcal{L}_{\lambda}\left(x_i; \bm{\theta}_{\mathcal{L}_\lambda} (\mathbf{S}^i)\right)-\mathcal{L}_{\lambda}\left(x_i; \bm{\theta}_{\mathcal{L}_\lambda} (\mathbf{S})\right)\right|
    \leq \frac{2 \eta^2}{\left(\Lambda_{\min }+2\lambda (2p-p^2)\right) n} .
\end{equation*}
\end{proposition}

\begin{proof}
Consider loss function with sparsity regularization from Eq.(\ref{eq:sps_final}), and we have:
\begin{equation*}
    \begin{aligned}
        \mathcal{L}_\lambda (\bm{\theta}) & = \mathcal{L}(\bm{\theta}) + \lambda \mathbb{E}_{\bm{d} \sim \mathrm{Bern}(2p-p^2)} ||\bm{d} \odot (\bm{\theta} - \bm{\theta}^0)||_2^2 \\
        & = \mathcal{L}(\bm{\theta}) + \lambda \mathbb{E}_{\bm{d} \sim \mathrm{Bern}(2p-p^2)} \sum_i{d_i^2(\theta_i - \theta_i^0)^2} \\ 
        & = \mathcal{L}(\bm{\theta}) + \lambda 
        \sum_i{(\theta_i - \theta_i^0)^2\mathbb{E}_{d_i \sim \mathrm{Bern}(2p-p^2)} d_i^2 } \\
        & = \mathcal{L}(\bm{\theta}) + \lambda \sum_i{(\theta_i - \theta_i^0)^2 (2p-p^2)} \\ 
        & = \mathcal{L}(\bm{\theta}) + \lambda (2p-p^2) ||\bm{\theta} - \bm{\theta}^0||_2^2.
    \end{aligned}
\end{equation*}
Through taking the results above to Lemma \ref{lem:apdx1}, we can substitute the regularization coefficient with $\lambda (2p-p^2)$ and obtain the pointwise hypothesis stability of LoRA Dropout algorithm, which is:
\begin{equation*}
    \mathbb{E}_{\mathbf{S}, i \sim \mathrm{U}(n)}\left|\mathcal{L}_{\lambda}\left(x_i; \bm{\theta}_{\mathcal{L}_\lambda} (\mathbf{S}^i)\right)-\mathcal{L}_{\lambda}\left(x_i; \bm{\theta}_{\mathcal{L}_\lambda} (\mathbf{S})\right)\right| \\
    \leq \frac{2 \eta^2}{\left(\Lambda_{\min }+2\lambda (2p-p^2)\right) n}.
\end{equation*}
    
\end{proof}

\subsection{Proof of Theorem \ref{theo:bayes_cls_error}} \label{apdx:proof_theo4.5}
\begin{theorem}[Error Bound of Bayes LoRA Dropout Ensemble]
    If the loss function $\mathcal{L}_\lambda$ is convex w.r.t. the model parameters $\bm{\theta}$, then we have:
    \begin{equation}
    \begin{aligned}
        \mathbb{E}_{(x,y)} \mathcal{L}_\lambda(\mathbb{E}_{\bm{\theta} \sim \mathcal{D}} \mathcal{M}(x; \bm{\theta}), y) \leq \mathbb{E}_{\bm{\theta} \sim \mathcal{D}} \mathbb{E}_{(x,y)} \mathcal{L}_\lambda(\mathcal{M}(x; \bm{\theta}), y).
    \end{aligned}
    \end{equation}
\end{theorem}

\begin{proof}
\begin{align}
    \mathbb{E}_{(x,y)} \mathcal{L}_\lambda(\mathbb{E}_{\bm{\theta} \sim \mathcal{D}} \mathcal{M}(x; \bm{\theta}), y) & \leq \mathbb{E}_{(x,y)} \mathbb{E}_{\bm{\theta} \sim \mathcal{D}} \mathcal{L}_\lambda(\mathcal{M}(x; \bm{\theta}), y)  \label{eq:pf6} \\ 
    & = \mathbb{E}_{\bm{\theta} \sim \mathcal{D}} \mathbb{E}_{(x,y)} \mathcal{L}_\lambda(\mathcal{M}(x; \bm{\theta}), y).
\end{align}
The inequality (\ref{eq:pf6}) holds for the Jensen inequality under the convexity of $\mathcal{L}_\lambda$ w.r.t. parameters $\bm{\theta}$.
\end{proof}

\section{Experimental Details}
\label{apdx:experimental_details}
\numberwithin{equation}{section}
\numberwithin{table}{section}
\subsection{Implementation Details of NLU Task}
All of our experiments on NLU task are implemented based on PyTorch 1.9.1 with Python 3.7.16 on the HuggingFace transformers library~\cite{wolf2019huggingface} 4.4.2. Fine-tuning is conducted on the pre-trained DeBERTaV3-base~\cite{he2021debertav3} model, and PEFT methods are applied on all the linear layers in every transformer block. 
We mainly follow the hyperparameter setting as ~\cite{zhang2023adaptive} and tune hyperparameters exclusive to our model.
The hyperparameters used when fine-tuning on each BLUE task are shown in Table \ref{tab:hyp_glue}.
For the hardware environment, We perform our experiments on a single NVIDIA-A100-80GB GPU or distributedly on 2 NVIDIA-RTX3090-24GB GPUs. 

\begin{table}[h]
  \caption{Summary of hyperparameter settings when fine-tuning on different tasks of the GLUE benchmark.}
  \vskip 0.1in
  \label{tab:hyp_glue}
\centering
\begin{tabular}{l|cccccccc}
\toprule
\textbf{Corpus} & \textbf{MNLI} & \textbf{RTE} & \textbf{QNLI} & \textbf{MRPC} & \textbf{QQP} & \textbf{SST-2} & \textbf{CoLA} & \textbf{STS-B} \\ \midrule
learning rate   & 5e-4          & 1.2e-3       & 5e-4        & 1e-3          & 5e-4         & 8e-4           & 5e-5          & 2.2e-3         \\
batch size      & 32            & 32           & 32            & 32            & 32           & 32             & 32            & 32             \\
\# epochs       & 7             & 50           & 5             & 30            & 10           & 24             & 25            & 25             \\ \midrule
dropout rate    & 0.5           & 0.5          & 0.5           & 0.5           & 0.5          & 0.5            & 0.5           & 0.5            \\
sample number   & 4             & 4            & 4             & 4             & 4            & 4              & 4             & 4              \\ \bottomrule
\end{tabular}
\end{table}

\subsection{Implementation Details of QA Task}

All of our experiments on QA task are implemented based on PyTorch 1.9.1 with Python 3.7.16 on the HuggingFace transformers library~\cite{wolf2019huggingface} 4.21.0. Fine-tuning is conducted on the pre-trained DeBERTaV3-base~\cite{he2021debertav3} model, and PEFT methods are applied on all the linear layers in every transformer block. 
We control the ratio of tunable parameters by adjusting the hyperparameters related to parameter budget, e.g. adapter dimension or LoRA rank. Specifically, the tunable parameter ratios of \{0.16\%,0.32\%,0.65\%\} correspond to LoRA rank of \{2,4,8\} respectively.
Other hyperparameters used when fine-tuning on SQuAD benchmark are shown in Table \ref{tab:hyp_squad}.
For the hardware environment, We perform our experiments on a single NVIDIA-A100-80GB GPU or distributedly on 2 NVIDIA-RTX3090-24GB GPUs. 

\begin{table}[h]
  \caption{Summary of hyperparameter settings when fine-tuning on the SQuAD benchmark.}
  \vskip 0.1in
  \label{tab:hyp_squad}
\centering
\begin{tabular}{l|cc}
\toprule
\textbf{Corpus} & \textbf{SQuAD v1.1} & \textbf{SQuAD v2.0} \\ \midrule
learning rate   & 1e-3                & 1e-3                \\
batch size      & 16                  & 16                  \\
\# epochs       & 10                  & 12                  \\ \midrule
dropout rate    & 0.5                 & 0.5                 \\
sample number   & 4                   & 4                   \\ \bottomrule
\end{tabular}
\end{table}

\subsection{Implementation Details of Instruction Tuning}
When performing instruction tuning, we use PyTorch 2.1.2 with Python 3.10.13. We employ the PEFT library~\cite{peft} and the LLaMA-Factory library~\cite{llama-factory} for implementing and evaluating our method. Fine-tuning is conducted on LLaMA2-7B~\cite{touvron2023llama}, and only the \{q\_proj,v\_proj,k\_proj,o\_proj\} linear modules in each transformer block get tuned. All hyperparameters used for fine-tuning LoRA and LoRA with dropout are shown in Table \ref{tab:hyper_llama2}. For the hardware environment, experiments are conducted distributedly on 2 NVIDIA-A100-80GB GPUs.

\begin{table}[h]
\caption{Summary of hyperparameter settings when fine-tuning LLaMA2-7B.}
\vskip 0.1in
\label{tab:hyper_llama2}
\centering
\begin{small}
\begin{tabular}{l|ccccccc}
\toprule
\textbf{Hyperparameter} & \textbf{lr} & \textbf{batch size}        & \textbf{rank} & \multicolumn{1}{l}{\textbf{lr-scheduler}} & \textbf{warmup step} & \multicolumn{1}{l}{\textbf{dropout rate}} & \multicolumn{1}{l}{\textbf{sample num}} \\ \midrule
\textbf{LoRA}           & 5e-5                   & 128                                        & 16                 & cosine                                    & 500                  & -                                         & -                                       \\
\textbf{LoRA+Dropout}   & 5e-5                   & 128                          & 16                 & cosine                                    & 500                  & 0.5                                       & 4                                       \\ \bottomrule
\end{tabular}
\end{small}
\end{table}

\section{DataSet Details}
\label{apdx:dataset_details}
\numberwithin{equation}{section}
\numberwithin{table}{section}
\subsection{Details of GLUE benchmark}
We use the General Language Understanding Evaluation (GLUE) benchmark~\cite{wang2018glue} for evaluation on NLU tasks. Following previous work~\cite{zhang2023adaptive}, eight datasets are picked for fine-tuning. Here we list detailed statistics of each dataset in Table \ref{tab:glue_data_statistic}.

\begin{table}[h]
\centering
  \caption{Summary of datset statistic of the GLUE benchmark.}
  \vskip 0.1in
  \label{tab:glue_data_statistic}
\begin{small}
\begin{tabular}{l|l|l|c|c|c|c}
\toprule
\textbf{Corpus} & \textbf{Task} & \textbf{Task Category} & \textbf{\#Train} & \textbf{\#Dev} & \textbf{\#Label} & \textbf{Metrics} \\ \midrule
\textbf{CoLA}                             & Acceptability                  & Single-Sentence Classification          & 8.5k                              & 1k                              & 2                                 & Matthews Corr                     \\ \midrule
\textbf{SST}                              & Sentiment                      & Single-Sentence Classification          & 67k                               & 872                             & 2                                 & Matched Accuracy                  \\ \midrule
\textbf{MNLI}                             & NLI                            & Pairwise Text Classification            & 393k                              & 20k                             & 3                                 & Accuracy                          \\ \midrule
\textbf{RTE}                              & NLI                            & Pairwise Text Classification            & 2.5k                              & 276                             & 2                                 & Accuracy                          \\ \midrule
\textbf{QQP}                              & Paraphrase                     & Pairwise Text Classification            & 364k                              & 40k                             & 2                                 & Accuracy                          \\ \midrule
\textbf{MRPC}                             & Paraphrase                     & Pairwise Text Classification            & 3.7k                              & 408                             & 2                                 & Accuracy                          \\ \midrule
\textbf{QNLI}                             & QA/NLI                         & Pairwise Text Classification            & 108k                              & 5.7k                            & 2                                 & Accuracy                          \\ \midrule
\textbf{STS-B}                            & Similarity                     & Text Similarity                         & 7k                                & 1.5k                            & 1                                 & Pearson Corr                      \\ \bottomrule
\end{tabular}
\end{small}
\end{table}

\subsection{Details of SQuAD benchmark}

The SQuAD (Stanford Question Answering Dataset) benchmark is a benchmark for question answering task collected from Wikipedia by crowd-workers. Specifically, the task is treated as a sequence labeling problem, where the probability of tokens from the start and end of the answer span are picked for prediction.
SQuAD v1.1\cite{rajpurkar2016squad} is the first version of SQuAD, including over 100,000 question-answer pairs sourced from 536 articles. And SQuAD v2.0\cite{rajpurkar2018know} adds 50,000 unanswerable questions written by humans based on SQuADv1.1. Therefore, SQuAD v2.0 further demands the model to be able to differentiate whether a question is unanswerable. Statistics of both SQuAD datasets are shown in Table \ref{tab:squad_data_statistic}.

\begin{table}[h]
\centering
\caption{Dataset statistic of the SQuAD benchmark.}
\vskip 0.1in
  \label{tab:squad_data_statistic}
\begin{tabular}{l|cc}
\toprule
\textbf{Corpus}     & \textbf{\#Train} & \textbf{\#Validation} \\ \midrule
\textbf{SQuAD v1.1} & 87,599           & 10,570                \\
\textbf{SQuAD v2.0} & 130,319          & 11,873                \\ \bottomrule
\end{tabular}
\end{table}

\subsection{Details of Alpaca dataset benchmark}
We fine-tune LLaMA2-7B on the Alpaca-clean dataset\footnote{https://huggingface.co/datasets/yahma/alpaca-cleaned}.
Alpaca-clean is the cleaned version of the original Alpaca dataset~\cite{alpaca}. It consists of 51K instructions and demonstrations and is suitable for instruction-tuning.
The cleaned version fixed multiple issues in the original release, including hallucinations, merged instructions, empty outputs, empty code examples, and instructions to generate images.

\end{document}